\documentclass[sigconf]{aamas}  

\AtBeginDocument{%
  \providecommand\BibTeX{{%
    \normalfont B\kern-0.5em{\scshape i\kern-0.25em b}\kern-0.8em\TeX}}}
    
\usepackage{booktabs}    
\usepackage{flushend} 

\setcopyright{ifaamas}  
\copyrightyear{2020} 
\acmYear{2020} 
\acmDOI{} 
\acmPrice{} 
\acmISBN{} 
\acmConference[AAMAS'20]{Proc.\@ of the 19th International Conference on Autonomous Agents and Multiagent Systems (AAMAS 2020)}{May 9--13, 2020}{Auckland, New Zealand}{B.~An, N.~Yorke-Smith, A.~El~Fallah~Seghrouchni, G.~Sukthankar (eds.)}  

\usepackage{tabularx} 
\usepackage{graphicx} 
\usepackage{cite} 
\usepackage{tcolorbox}
\usepackage{amsmath}
\usepackage{amsthm}
\usepackage{algorithm}
\usepackage{algorithmic}
\usepackage{subcaption}
\usepackage{bbm}
\usepackage{multirow}
\usepackage{tabularx} 
\usepackage{amsmath}  
\usepackage{mathtools}
\usepackage{bm}

\newtheorem{remark}{Remark}

\usepackage{mathtools}
\DeclarePairedDelimiter{\ceil}{\lceil}{\rceil}

\begin{document}

\title{Private and Byzantine-Proof Cooperative Decision-Making}  



%
\author{Abhimanyu Dubey}
\affiliation{%
 \institution{Massachusetts Institute of Technology}
}
\email{dubeya@mit.edu}

\author{Alex Pentland}
\affiliation{%
 \institution{Massachusetts Institute of Technology}
}
\email{pentland@mit.edu}
%
%
%
%
%
%

\begin{abstract}  
The cooperative bandit problem is a multi-agent decision problem involving a group of agents that interact simultaneously with a multi-armed bandit, while communicating over a network with delays. The central idea in this problem is to design algorithms that can efficiently leverage communication to obtain improvements over acting in isolation. In this paper, we investigate the stochastic bandit problem under two settings - (a) when the agents wish to make their communication private with respect to the action sequence, and (b) when the agents can be byzantine, i.e., they provide (stochastically) incorrect information. For both these problem settings, we provide upper-confidence bound algorithms that obtain optimal regret while being (a) differentially-private and (b) tolerant to byzantine agents. Our decentralized algorithms require no information about the network of connectivity between agents, making them scalable to large dynamic systems. We test our algorithms on a competitive benchmark of random graphs and demonstrate their superior performance with respect to existing robust algorithms. We hope that our work serves as an important step towards creating distributed decision-making systems that maintain privacy.
\end{abstract}

\keywords{multi-armed bandits; sequential decision-making; differential privacy; robustness}  

\maketitle


\section{Introduction}
Cooperative decision-making in multiagent systems has been a prominent area of scientific inquiry, with interest from a variety of disciplines from robotics and control systems to macroeconomics and game theory. Applications of cooperative decision-making range from robotics~\citep{cao1997cooperative, kok2003multi} to sensor networks~\citep{byers2000utility,klein2004sensor}.

However, as the availability of large-scale dense data (often collected through a network of sources) increases, the problem of cooperative learning among multiple agents becomes increasingly relevant, moving beyond systems that are ``effectively'' single-agent to much larger, real-time systems that are decentralized and have agents operating independently. Examples of such applications include seamless personalization across multiple devices in IoT networks~\citep{kisch2017real, roman2013features, wang2015processing}, that constantly share data between agents.

In such a framework, it is imperative to ensure that information is securely and robustly shared between agents, and individual concerns around privacy are not violated in the pursuit of performance on inference problems. Most generally, the learning setting assumed in such problems and applications is sequential decision-making, where observations are provided in a sequence, and each agent takes actions with the knowledge of this stream of information.

The multi-armed bandit~\citep{thompson1933likelihood} is a classic framework to investigate sequential decision-making in. It provides an elegant mathematical framework to make precise statements about the \textit{exploration-exploitation} dilemma that is central to sequential decision-making, and has thus become the environment of choice across a wide variety of application domains including, most prominently, online advertising~\citep{gentile2014online}. An extension is the \textit{cooperative} multiagent bandit, where a group of agents collectively interact with the same decision problem, and must \textit{cooperate} to obtain optimal performance.

However, in a decentralized environment, cooperative decision-making comes with several caveats. Since it is impossible to control any individual agents' behavior from a centralized server, enforcing important constraints such as private computaiton with respect to observations becomes a difficult technical challenge. Additionally, if any agent is malicious (i.e. provides incorrect or adversarial communication), achieving optimal performance becomes non-trivial.

In this paper, we consider two problems in the context of cooperative multiagent decision-making. In the multi-armed (context-free) bandit environment, we investigate (a) privacy in inter-agent communication and (b) tolerance to \textit{byzantine} or malicious agents that deliberately provide stochastically corrupt communication. Our contributions can be listed as follows.

1. We provide a multi-agent UCB-like algorithm \textsc{Private Multiagent UCB} for the cooperative bandit problem, that guarantees communication of the reward sequence of any agent to be private to any other receiving agent, and provides optimal group regret. Our algorithm is completely decentralized, i.e., its performance or operation does not require any knowledge of the agents' communication network, and each agent chooses actions autonomously (i.e., without a central server). It also maintains privacy of the source agents' messages regardless of the individual behavior of any receiving agent (i.e., it is robust to defection in individual behavior).

2. We provide a multi-agent UCB-like algorithm, dubbed \textsc{Byzantine-Proof Multiagent UCB} for the cooperative bandit problem with \textit{byzantine} agents, that corrupt their messages following Huber's $\epsilon$-contamination model, and provide optimal (in terms of communication overhead) group regret bounds for its performance. Like our private algorithm, this too is completely decentralized, and operates without any knowledge of the communication network.

3. We validate the theoretical bounds on the group regret of our algorithm on a family of random graphs under a variety of different parameter settings, and demonstrate their superiority over existing single-agent algorithms and robustness to parameter changes.

The paper is organized as follows. First, we provide an overview of the background and preliminaries essential to our problem setup. Next, we examine the private setting, followed by the byzantine setting. We then provide results on our experimental setup and survey related work in this problem domain before closing remarks.
\section{Background and Preliminaries}
\textbf{The Cooperative Stochastic Bandit}. The multi-armed stochastic bandit problem is an online learning problem that proceeds in rounds. At every round $t \in \{1, 2, ..., T\}$ (denoted in shorthand as $[T]$), the learner selects an arm $A_t \in [K]$ and the bandit draws a random reward $X_t$ from a corresponding reward distribution $\nu_{A_t}$ with mean $\mu_{A_t}$. We assume the rewards are drawn from probability densities with bounded support in $[0, 1]$\footnote{This can be extended to any bounded support by renormalization, and we choose this setup for notational convenience.}, a typical assumption in the bandit literature. The objective is to minimize the \textit{Regret} R(T).
\begin{equation}
    R(T) = T\cdot\max_{k \in [K]}\mu_k - \sum_{t \in [T]}\mu_{A_t}.
\end{equation}
The cooperative multi-armed bandit problem is a distributed extension of the stochastic bandit problem. In this setting, a group of $M$ agents that communicate over a (connected, undirected) network $\mathcal G = (V, E)$ face an identical stochastic $K$-armed bandit, and must cooperate to collectively minimize their group regret $R_{\mathcal G}(T)$.
\begin{equation}
    R_{\mathcal G}(T) = TM\max_{k \in [K]}\mu_k - \sum_{m \in [M]}\sum_{t \in [T]}\mu_{A_{m,t}}.
\end{equation}
Here, we denote the action taken by learner $m$ at time $T$ as $A_{m, t}$ and the corresponding reward received as $X_{m, t}$, $d(m, m')$ denotes the length of the shortest path between agents $m$ and $m'$. For any $T$, $n_k(T)$ denotes the total number of times any arm $k$ is pulled across all agents, and $n^m_k(T)$ denotes the times agent $m$ has pulled the arm. In the cooperative bandit setting, we denote the policy of each agent $m \in [M]$ by $(\pi_{m, t})_{t \in [T]}$. The collective policy for all agents is denoted by $\Pi = (\pi_{m, t})_{m \in [M], t \in [T]}$. Since $\mathcal G$ is not assumed to be complete (i.e. each agent can communicate with every other agent), it introduces a delay in how information flows between agents. However, to provide an understanding of the limits of cooperative decision-making, consider the scenario  when $\mathcal G$ is complete. In this setting, after every trial of the problem, agents can instantaneously communicate all required information, and therefore the problem can effectively be thought of as a centralized agent pulling $M$ arms at every trial. Anantharam~\textit{et al.}\citep{anantharam1987asymptotically} provide a lower bound on the group regret achievable in this setting, given as follows. 
\begin{theorem}[Lower Bound on Centralized Estimation\citep{anantharam1987asymptotically}]
For the stochastic bandit with $K$ arms, each with mean $\mu_k$, let $\Delta_k = \max_{k \in [K]}\mu_k - \mu_k$ and $\textsf{D}(\mu_*, \mu_k)$ denote the Kullback-Liebler divergence between the optimal arm $k^* = \arg\max_{k \in [K]} \mu_k$ and arm $k$. Then, for any algorithm that pulls $M$ arms at every trial $t \in [T]$, the total regret incurred satisfies
\begin{equation*}
    \liminf_{T\rightarrow\infty}R(T) \geq \left(\sum_{k:\Delta_k > 0} \frac{\Delta_k}{\textsf{D}(\mu_*, \mu_k)}\right)\ln T.
\end{equation*}
\end{theorem}
This implies that the per-agent regret in this case is $\mathcal O(M^{-1}\ln T)$. While the above result holds for centralized estimation, we see that one cannot hope for (a) a dependence on $T$ that grows sublogarithmically, and (b) a per-agent regret with a better dependence on $M$ than $\mathcal O(M^{-1})$, since delays will only limit the amount of information possessed by any agent. We will demonstrate that this bound can be matched (up to additive constants, in some algorithmic settings) in the two problems we consider, for any connected $\mathcal G$.

\textbf{Communication Protocol}. While the existing work on the cooperative bandit problem assumes communication via a gossip or consensus protocol~\citep{landgren2016distributed, landgren2016distributed2}, this protocol requries complete knowledge of the graph $\mathcal G$, which is not feasible in a decentralized setting. In this paper, we therefore adopt a variant of the \textsf{Local} communication protocol that is utilized widely in distributed optimization~\citep{moallemi2007message}, asynchronous online learning~\citep{suomela2013survey} and non-stochastic bandit settings~\citep{cesa2019delay, bar2019individual}. Under this protocol, at any time $t$, an agent $m \in [M]$ creates a message $\bm q_m(t)$ that it broadcasts to all its neighbors in $\mathcal G$. $\bm q_m(t)$ is a tuple of the following form.
\begin{equation}
    \bm q_m(t) = \left\langle m, t, f_1(H_m(t)), ..., f_L(H_m(t))\right\rangle.
\end{equation}
Here $H_m(t)$ denotes the history of action-reward pairs obtained by the agent until time $t$, i.e. $H_m(t) = \left\{X_{m, 1}, A_{m, 1}, ..., X_{m, t}, A_{m, t}\right\}$, and $f_1, ..., f_L$ are real-valued functions for some message length $L > 0$. When a message is created, it forwarded from agent-to-agent $\gamma$ times before it is discarded. Hence, at any time $t$, the oldest message being propagated in the network would be one created at time $t-\gamma$. As long as a message has not already been received by an agent previously, the agent will broadcast the messages it receives (including its own) to all its neighbors. Under this protocol, we see that two agents that are a distance $\gamma$ apart can communicate messages to each other, and they will receive messages from each other at a delay of $\gamma$ trials. An important concept we will use in the remainder of the paper, based on this protocol is the \textit{power graph} of order $\gamma$ of $\mathcal G$, denoted as $\mathcal G_\gamma$, which is a graph that contains an edge $(i, j)$ if there is a path of length at most $\gamma$ between $i$ and $j$ in $\mathcal G$.

In the single-agent setting, the policy of the learner $\pi = (\pi_t)_{t \in [T]}$ is a sequence of functions $\pi_t$ that define a probability distribution over the space of actions, conditioned on the history of decision-making of the learner, i.e. $\pi_t = \pi_t(\cdot \in [K] | X_1, A_1, ..., X_t, A_t)$. In the decentralized multi-agent setting, we have a set of policies $\Pi = (\pi_{m})_{m \in [M]}$ for each agent $m \in [M]$, where $\pi_m = (\pi_{m, t})_{t \in [T]}$ defines a probability distribution over the space of actions conditioned on both the history of the agents and the messages it has received from other agents until time $t$. If for any pair of agents $m, m'$ in $\mathcal G$, $d(m, m')$ denotes the distance between the two agents in $\mathcal G$, the policy for agent $m$ at time $t$ is of the following form.
\begin{equation}
    \pi_{m, t} = \pi_{m, t}\left(\cdot \in [K] | H_m(t) \cup_{m' \in \mathcal N(m, \mathcal G_\gamma), t \in T} \left(q_{m'}\left(t - d(m, m')\right)\right)\right)
\end{equation}
Here $\mathcal N_\gamma(m)$ denotes the $\gamma$-neighborhood of agent $m$ in $\mathcal G$, and use the short notation $I_m(t) = \cup_{m' \in \mathcal N_\gamma(m), t \in T} \left(q_{m'}\left(t - d(m, m')\right)\right)$ to denote the set of all messages possessed by the agent $m$ at time $t$. We will now describe the first class of algorithms that are designed for privacy in the cooperative setting.
\section{Private Cooperative Bandits}
We begin by first providing an overview of $\epsilon$-differential privacy and the techniques we will use in our algorithm.
\subsection{Differential Privacy}
Differential privacy, proposed by Dwork~\citep{dwork2011differential} is a formalisation of the amount of information that is leaked about the \textit{input} of an algorithm to an adversary observing its \textit{output}. Existing work in the single agent~\citep{tossou2016algorithms, mishra2015nearly} and the (centralized) multi-agent~\citep{tossou2016algorithms} bandit setting has assumed the algorithms only to have actions (outputs) that are differentially-private with respect to their rewards (inputs). However, since we are considering a decentralized setting, we assume here that each learner acts in isolation, while communicating with other agents by sharing messages. Its policy, therefore, is a function of its own decision history and the messages it receives from other agents. By making the \textit{messages} differentially-private with respect to the reward sequence an agent observes, we can ensure that any policy using these messages is also private with respect to the rewards obtained by other agents. We use this to define differential privacy in the decentralized bandit context.
\begin{definition}[$\epsilon$-differentially private message protocol]
A message $\bm q_m(t)$ composed of $L$ functions $(f_i)_{i \in [L]}$ for agent $m$ at time $t$ is $\epsilon$-differentially private with respect to the rewards it obtains if for all histories $H_m(t)$ and $H'_{m}(t)$ that differ in at most one reward sample, we have for all $i \in [L], S \subseteq \mathbb R$:
\begin{equation*}
    e^{-\epsilon} \leq \frac{\text{Pr}\left(f_i \in S | H_m(t)\right)}{\text{Pr}\left(f_i \in S | H'_m(t)\right)} \leq e^\epsilon
\end{equation*}
\end{definition}
There are several advantages of having such a mechanism to introduce privacy in the cooperative decision-making process. First, we see that unlike centralized or single-agent bandit algorithms that inject noise as a part of the central algorithm itself, our definition \textit{requires} each message to \textit{individually} be private, regardless of the recipient agents' algorithm. Such a protocol ensures that if any learner decides to reveal information about the reward sequence obtained by its neighbors, it will not be able to do so. In the later sections we will demonstrate that for any agent, guaranteeing the differential privacy of the messages it generates with respect to its own rewards suffices to ensure that our algorithm is differentially private with respect to the rewards of other agents. We now introduce concepts crucial to the development of our algorithm. 

\textbf{The Laplace Mechanism}. One of the main techniques to introduce differential privacy for univariate problems is the Laplace mechanism, that operates on the notion of \textit{sensitivity}. For any domain $\mathcal U$ and function $f : \mathcal U^* \rightarrow \mathbb R$, we can define \textit{sensitivity} of $f$ for two datasets $D, D' \in \mathcal U^*$ that are neighbors (differ only in one entry) as follows.
\begin{equation}
    s(f) = \max_{D, D' \in \mathcal U^*}| f(D) - f(D') |
    \label{eqn:sensitivity}
\end{equation}
The Laplace mechanism operates by adding zero-mean noise to the output of $f$ with a scale that is governed by the level of privacy required and the sensitivity of $f$, as determined by the following.
\begin{lemma}[Theorem 4 of Dwork and Roth~\citep{dwork2014algorithmic}]
For any real function $f$ of the input data $D$, adding a Laplace noise sample $X$ with scale parameter $\beta$ ensures that $f(D) + X$ is $\left(\beta/s(f)\right)$-differentially private with respect to $D$, where $s(f)$ is the sensitivity of $f$.
\end{lemma}
The Laplace mechanism is the underlying approach that we utilize in guaranteeing privacy in the multi-agent setting. We will now describe our message-passing protocol, which uses this mechanism to guarantee differential privacy between any pair of agents.
\subsection{Private Message-Passing}
The general approach to privacy in the (single-agent) bandit setting involves using a binary tree or hybrid mechanism to compute the running sum of rewards for any arm~\citep{mishra2015nearly, tossou2016algorithms}. This mechanism involves maintaining a binary tree of individual rewards, where each node in the tree stores a partial sum. This implies that updating the tree would be logarithmic in the number of elements present, and a similar (logarithmic) bound on the sensitivity can be derived, therefore, for any arm $k$ that has been pulled $m$ times, one requires introducing Laplace noise $\mathcal L(\frac{\ln m}{\epsilon})$ to achieve $\epsilon$-differential privacy. Under this mechanism,~\citep{tossou2016algorithms} provide a bound on the regret.
\begin{theorem}[Regret of \textsc{DP-UCB} Algorithm~\citep{tossou2016algorithms}]
The single-agent \textsc{DP-UCB} algorithm, when run for $T$ trials, obtains regret of $ O\left(K\ln T\cdot\max\left\{\frac{4\sqrt{8}}{\epsilon}\ln\left(\ln T\right), \frac{8}{\Delta_{\min}}\right\}\right)$.
\end{theorem}
While this approach is feasible in the single-agent case, in the distributed setting, we will have to maintain $M$ separate binary trees (one for each agent), which would create an overhead of a factor of $M$ in the regret (since the sensitivity is increased by a factor of $M$). To mitigate this overhead we can alternatively utilize the \textit{interval} mechanism, introduced in~\citep{tossou2016algorithms} to the multi-agent setting. Under this mechanism, the mean of an arm is updated only (approximately) $T/\epsilon$ times, which makes it possible to add Laplacian noise that is of a lower scale, greatly improving the regret obtained by the algorithm:~\citep{tossou2016algorithms} demonstrate that this mechanism obtains the optimal \textit{single-agent} regret with an \textit{additive} increase due to differential privacy. We demonstrate that using a message-passing algorithm that is inspired by the interval mechanism, we obtain a group regret that has a much smaller than $\mathcal O(M)$ dependence on the number of agents, and is additive in terms of the privacy overhead.

For the private setting, in the message-passing protocol described in the earlier sections, consider for any agent $m$, the following message created at time $t \in [T]$.
\begin{equation}
    \bm q_m(t) = \left\langle m, t, \gamma, \hat{\bm{v}}_m(t), \bm{n}_m(t) \right\rangle
\end{equation}
Here, $\hat{\bm{v}}_m(t) = (\hat v_k^m(t))_{k \in [K]}$ is a vector of arm-wise reward means that comprise the empirical mean of rewards with specific Laplace noise added, and $\bm{n}_m(t) = (n_k^m(t))_{k \in [K]}$ is a vector of the number of times an arm $k$ has been pulled by the agent $m$. Since the interval mechanism proceeds by making the updates to the mean estimate infrequent, this strategy is adopted in our message-passing protocol as well, which is described in the Lemma below. We describe the complete message- passing protocol in Algorithm~\ref{alg:privacy_protocol}.
\begin{lemma}[Private release of means~\citep{tossou2016algorithms}]
The interval $w$ to update the mean for each arm follows a series such that for the given value of $\epsilon \in (0, 1), v \in (1, 1.5), w_n = W_{n+1} - W_n$ with $W_0 = 0$, and $W_{n+1}$ such that,
\begin{equation*}
    W_{n+1} = \inf_{x' \in \mathbb N}\left\{ x' \geq W_n + 1 : \sum_{W_n + 1}^{x'} \frac{1}{\sqrt{i^v}} \geq \frac{1}{\epsilon y'^{v}}\right\}.
\end{equation*}
\label{lem:update_interval}
\end{lemma}

\textbf{Protocol Description}: The protocol followed is straightforward: at any time $t$, the agent updates the sum of its own rewards with the reward obtained from the bandit algorithm for the arm pulled. Then, for a set of intervals $w_0, ..., w_T$ as defined in Lemma~\ref{lem:update_interval}, the broadcasted version of this personal mean is updated to the latest version, with an additional Laplace noise $\mathcal L(n_k^m(t)^{v/2-1})$, for some constant $v \in (1, 1.5)$. This mechanism ensures that the broadcast of the personal mean is differentially-private to the reward sequence observed, for each arm. After this update procedure has concluded for each arm, the agent broadcasts this (and all previously incoming messages from other agents) to its neighbors using the system routine \textsc{SendMessages}, which is assumed to be constant-time. It then updates its own group mean estimates of the mean for each arm $(\hat\mu_k^m(t))$ using the incoming messages from all other agents. This is done by using the latest message from each agent in $\mathcal N_\gamma(m)$. It then discards stale messages (with life parameter $l=0$), and returns the updated group means for all arms for the bandit algorithm to use. We now describe the privacy guarantee for the messages, and then describe the $\textsc{Private-Multi-UCB}$ algorithm and its privacy and regret guarantees.
\begin{algorithm}[t]
\caption{\textsc{Private Message-Passing Protocol}}
\label{alg:privacy_protocol}
\begin{algorithmic}[1] 
\STATE \textbf{Input}: Agent $m \in [M]$, Iteration $t \in [T]$, series $W = (w_0, w_1, ... w_T)$, such that $W(i) = w_i$, series counter $i^m_k$ for each $k \in [K]$; $i^m_k = w_0$ if $t=0 \ \forall k$, set of existing messages $Q_m(t-1)$, $Q_m(t) = \phi$ if $t=0$, privacy constant $v \in (1, 1.5)$, existing reward sums $s_{m, k}(t) \ \forall k \in [K]$ such that $s_{m, k}(0) = 0 \forall k$.
\STATE Obtain $X_{m, t}, A_{m, t}$ from bandit algorithm.
\STATE Set $s_{m, k}(t) = s_{m, k}(t-1) + X_{m, t}$ for $k = A_{m, t}$.
\STATE Set $n_k^m(t) = n_k^m(t-1) + 1$ for $k = A_{m, t}$.
\FOR{Arm $k$ in $[K]$}
\IF{$n_k^m(t) = W(i_k^m)$}
\STATE Set $\hat v_m^k(t) = s_{m, k}(t)/n_k^m(t) + \mathcal L(n_k^m(t)^{v/2-1})$.
\STATE Set $i_k^m = i_k^m + 1$.
\ELSE
\STATE Set $\hat v_m^k(t) = v_m^k(t-1)$.
\ENDIF
\ENDFOR
\STATE Set $q_m(t) = \left\langle m, t, \hat{\bf{v}}_m(t), \bm{n}_m(t) \right\rangle$.
\STATE Set $Q_m(t) = Q_m(t-1) \cup \{q_m(t)\}$.
\FOR{Each neighbor $m'$ in $\mathcal N_1(m)$}
\STATE \textsc{SendMessages}$(m, m', Q_m(t))$.
\ENDFOR
\FOR{Each neighbor $m'$ in $\mathcal N_1(m)$}
\STATE $Q' = $\textsc{ReceiveMessages}$(m',m)$
\STATE $Q_m(t)$ = $Q_m(t) \cup Q'$.
\ENDFOR
\STATE Set $N^m_k(t) = n^m_k(t), \hat\mu^m_k(t) = s_{m, k}(t) \ \forall k \in [K]$.
\FOR{$q' = \langle m',t',x'_1, ..., x'_K, a'_1, ..., a'_K \rangle \ \in Q_m(t)$}
\IF{\textsc{IsLatestMessage}$(q')$}
\FOR{Arm $k \in [K]$}
\STATE Set $N^m_k(t) = N^m_k(t) + a'_k$.
\STATE Set $\hat\mu^m_k(t) = \hat\mu^m_k(t) + a'_k\cdot x'_k$.
\ENDFOR
\ENDIF
\ENDFOR
\FOR{Arm $k \in [K]$}
\STATE $\hat\mu^m_k(t) = \hat\mu^m_k(t)/N^m_k(t)$.
\ENDFOR
\end{algorithmic}
\end{algorithm}
\begin{lemma}[Privacy of Outgoing Messages]
For any agent $m \in [M], t\in [T], v \in (1, 1.5)$, each outgoing message $q_m(t)$ is $n_k^m(t)^{-v/2}$-differentially private w.r.t. the sequence of arm $k \in [K]$.
\end{lemma}
\begin{proof}
This follows directly from the fact that the only element that is dependent on the reward sequence of arm $k$ is its noisy mean $\hat\mu_k^m(t)$, which is $n_k^m(t)^{-v/2}$- differentially private since we add Laplace noise of scale $n_k^m(t)^{v/2-1}$ and that the sensitivity of the empirical mean is $n_k^m(t)^{-1}$.
\end{proof}
\subsection{Private Multi-Agent UCB}
The private multi-agent UCB algorithm is straightforward. During a trial, any agent $m \in [M]$ constructs the group means for each arm using the reward samples from its own pulls and the pulls obtained from its neighbors in $\mathcal N_\gamma(m)$. This mean is given by the estimate $\hat\mu_k^m(t)$ for any arm $k$, as obtained in Algorithm~\ref{alg:privacy_protocol}. Using these samples, it then follows the popular UCB strategy, where it constructs an upper confidence-bound for each arm, and pulls the arm with the largest upper confidence bound. The complete algorithm is described in Algorithm~\ref{alg:private_multi_ucb}. First, we describe the privacy guarantee with respect to its neighboring agents in $\mathcal N_\gamma(m)$, and then prove the regret (performance) guarantee on the group regret.

\begin{algorithm}[t]
\caption{\textsc{Private Multi-Agent UCB}}
\label{alg:private_multi_ucb}
\begin{algorithmic}[1] 
\STATE \textbf{Input}: Agent $m \in [M]$, trial $t \in [T]$, arms $k \in [K]$, mean $\hat\mu_k^m(t)$ and count $n_k^m(t)$ estimates  for each arm $k \in [K]$, from Algorithm~\ref{alg:privacy_protocol}.
\IF{$t \leq K\ceil*{1/\epsilon}$}
\STATE $A_{m, t} = t \mod K$.
\ELSE
\FOR{Arm $k \in [K]$}
\STATE $\text{UCB}_k^{m}(t) = \sqrt{\frac{2\ln t}{N_k^m(t)}}$.
\ENDFOR
\STATE $A_{m, t} = \arg\max_{k \in [K]} \left\{\hat\mu_k ^{m}(t) + \text{UCB}_k^{(m)}(t)\right\}$.
\ENDIF
\STATE $X_{m, t} = \textsc{Pull}(A_{m, t})$.
\RETURN $A_{m, t}, X_{m, t}$.
\end{algorithmic}
\end{algorithm}
\begin{lemma}[Privacy Guarantee]
After $t$ trials, any agent $m \in [M]$ is $(e'_{m'}, \delta'_m)$-differentially private with respect to the reward sequence of any arm observed by any other agent $m' \in \mathcal N_\gamma(m)$ communicating via Algorithm~\ref{alg:privacy_protocol} with privacy parameter $v \in (1, 1.5)$, where, for $\epsilon \in (0, 1], \delta'_m \in (0, 1]$, $\epsilon'_{m'}$ satisfies
\begin{equation*}
    \epsilon'_m \leq \min \left(\epsilon \frac{(t-d(m, m'))^{1-v/2}}{1-v/2}, 2\epsilon\zeta(v) + \sqrt{2\epsilon \zeta(v)\ln(1/\delta')}\right).
\end{equation*}
\end{lemma}
\begin{proof}
We can see that for any arm $k$, the estimate $\hat\mu_k^m(t)$ is composed of the sum of rewards from all other agents in $\mathcal N_\gamma(m)$. However, with respect to the reward sequence of any single agent $m' \in \mathcal N_\gamma(m)$, this term only depends on the differential privacy of the outgoing messages $q_{m'}$ obtained by agent $m$ until time $t-d(m, m')$ (since it takes at least $d(m, m')$ trials for a message from $m'$ to reach $m$). However, since for each arm, a new mean is released only $t/f$ times (where $f$ is the interval in Algorithm~\ref{alg:privacy_protocol}), a $k$-fold composition theorem~\citep{dwork2010boosting} identical to~\citep{tossou2016algorithms} can be applied with $t' = t-d(m, m')$. Using Theorem 3.4 and then Corollary 3.1 of~\citep{tossou2016algorithms} subsequently provides us the result.
\end{proof}
\begin{remark}[Robustness to Strategy]
It is important to note that if all agents $m \in [M]$ follow the protocol in Algorithm~\ref{alg:privacy_protocol}, then the privacy guarantee is sufficient, regardless of the algorithm any agent individually uses to make decisions. This is true since for any agent, the the complete sequence of messages it receives from any other agent is differentially private with respect to the origin's reward sequence for any arm and at any instant of the problem. Hence, if the agent chooses to ignore other agents and follow another an arbitrary decision-making algorithm, the privacy guarantee only gets stronger, never worse (since it cannot gain any additional information from the messages). Therefore, this communication protocol is robust to any agent's individual decision-making policy.
\end{remark}
\begin{remark}[Robustness to Defection]
The communication protocol crucially involves broadcasting rewards with privacy alterations at the source, and not at the destination. Consider an alternate protocol where the agents broadcasted their true personal average rewards (without noise), and the receiving agent added an appropriately scaled Laplacian noise on receiving a message. While this protocol would yield slightly better guarantees (both on privacy and regret, since the noise will be one Laplace sample, and not the sum of several samples in the current case), it is not robust to defection i.e. if the agent decides to learn the individual reward sequence, it can. Since our mechanism adds noise at the source, it maintains robustness under defection.
\end{remark}
\begin{theorem}
If all agents $m \in [M]$ each follow the \textsc{Private Multi-Agent UCB} algorithm with the messaging protocol described in Algorithm~\ref{alg:privacy_protocol}, then the group regret incurred after $T$ trials obeys the following upper bound.
\begin{multline*}
    R_{\mathcal G}(T) \leq \sum_{k : \Delta_k > 0} \bar\chi(\mathcal G_\gamma)\left(\frac{8\ln T}{\Delta_k}\right) + \\ \left(\sum_{k : \Delta_k > 0} \Delta_k\right)\left(\chi(\mathcal G_\gamma)\left(M\gamma + 2\right) + M\left(\frac{1}{\epsilon} + \zeta(1.5)\right)\right).
\end{multline*}
Here, $\bar\chi$ is the clique covering number, and $\zeta$ is the Riemann zeta function.
\end{theorem}
\begin{proof}(Sketch) 
The proof proceeds by decomposing $\mathcal G_\gamma$ into its minimum clique partitions, and analysing each clique individually. We then derive a UCB-like decomposition for the group regret of each clique. We first bound the delay in information propagation through this clique, and then, using concentration results for bounded and Laplace random variables, we bound the number of pulls of a suboptimal arm within the clique, and subsequently bound the total group regret. A full proof is presented in the Appendix.
\end{proof}
\begin{remark}[Deviation from Optimal Rates]
Upon examining the bound, we see that there are two $O(1)$ terms - the first, $O(\bar\chi(\mathcal G_\gamma)M\gamma)$ is the loss from delays in communication, since information flows between two agents with delays, and path lengths between agents are (on average) larger than 1 (unless $\mathcal G$ is connected). When $\mathcal G$ is connected, this reduces to a constant $M+2$, that matches the single-agent UCB case for $MT$ total pulls. The second term, $M\left(\frac{1}{\epsilon} + \zeta(1.5)\right)$, arises from the level of differential privacy required. When $\epsilon = 1$, i.e. no privacy, this term also reduces to a constant of $M$, matching the single-agent non-private case with $MT$ pulls. Finally, consider the overhead in the logarithmic term, $\chi(\mathcal G_\gamma)$. This too, arises from the delays in the network. When $\mathcal G$ is complete, this reduces to 1, matching the optimal lower bound.
\end{remark}
\begin{remark}[Dependence on $\gamma$ and Asymptotic Optimality]
In addition to the dependencies discussed earlier, a major factor in controlling the asymptotic optimality is via the life parameter $\gamma$, that controls ``how far'' information flows from any single agent. This dependence is present within the bound as well, as we can tune the performance of the algorithm greatly by tuning $\gamma$. If we set $\gamma = \textsf{diameter}(\mathcal G)$, then $\chi(\mathcal G_\gamma) = 1$ (since $\mathcal G_\gamma$ will be fully connected), which implies that the group regret incurred asymptotically ($\lim T \rightarrow \infty$) matches the lower bound. Alternatively, if we were to run the algorithm with $\gamma = 0$, i.e. no communication, then we obtain a regret of $\mathcal O(M\ln T)$, that is equivalent to the group regret obtained by $M$ agents operating in isolation. Interestingly, there is no dependence of $\mathcal G$ and $\gamma$ on the privacy loss. This is intuitive, since our mechanism introduces suboptimality at the source, and not at the destinations.
\end{remark}
We defer experimental evaluations to the end of the paper. We now describe the algorithm and message-passing protocol for the setup with byzantine agents.
\section{Byzantine-Proof Bandits}
In this section, we examine the cooperative multi-agent bandit problem where agents cannot be trusted. In this setting, there exist \textit{byzantine} agents, that, at any trial, instead of reporting their true means, give a random sample from an alternate (but fixed) distribution, with some probability $\epsilon$. Once a message is created, however, we assume that it is received correctly by all subsequent agents, and any corruption occurs only at the source.

One might consider the setting where in addition to messages being corrupted at the source, it is possible for a message to be corrupted by any intermediary byzantine agent with the same probability $\epsilon$. From a technical perspective, this setting is not different than the first setting, since the probability of any incoming message being corrupted becomes at most $\gamma\epsilon$ (by the union bound and that $d(m, m') \leq \gamma$ for any pair of agents $m, m'$ that can communicate), and the remainder of the analysis is identical henceforth. Therefore, we simply consider the first setting. We now formalize the problem and provide some technical tools for this setting.

\textbf{Problem Setting and Messaging Protocol}. For the byzantine-proof setting, in the message-passing protocol described in the earlier sections, consider for any agent $m$, the following message created at time $t \in [T]$.
\begin{equation}
    \bm q_m(t) = \left\langle m, t, A_{m, t}, \hat X_{m, t}\right\rangle
\end{equation}
Where, for byzantine agents, $\hat X_{m, t} = X_{m, t}$ with probability $(1-\epsilon)$, and a random sample from an unknown (but fixed) distribution $Q$ with probability $\epsilon$. For non-byzantine agents, $\hat X_{m, t} = X_{m, t}$ with probability 1. Contrasted to the protocol used in the previous problem setting, there are several key differences. First, instead of communicating reward means from all arms, we simply communicate just the individual rewards obtained along with the action taken. This is essential owing to the nature of our algorithm, as will be made precise later on (cf. Remark~\ref{remark:doing_both}).
\subsection{Robust Estimation}
The general problem of estimating statistics of a distribution $P$ from samples of a mixture distribution of the form $(1-\epsilon)P + \epsilon Q$ for $\epsilon < 1/2$ is a classic contamination model in robust statistics, known as Huber's $\epsilon$-Contamination model~\citep{huber2011robust}. In the univariate mean estimation setting, a popular approach is to utilize a \textit{trimmed} estimator of the mean that is robust to outlying samples, that works as long as $\epsilon$ is small. We utilize the trimmed estimator to design our algorithm as well, defined as follows.
\begin{definition}[Robust Trimmed Estimator]
Consider a set of $2N$ samples $X_1, Y_2, ..., X_{N}, Y_N$ of a mixture distribution $(1-\epsilon)P + \epsilon Q$ for some $\epsilon \in [0, 1/2)$. Let $X^*_1, \leq X^*_2 \leq ... \leq X^*_{N}$ denote a non-decreasing arrangement of $X_1, ..., X_N$. For some confidence level $\delta \in (0, 1)$, let $\alpha = \max(\epsilon, \frac{\ln(1/\delta)}{N})$, and let $\mathcal Z$ be the shortest interval containing $N\left(1-2\alpha-\sqrt{2\alpha\frac{\ln(4/\delta)}{N}} - \frac{\ln(4/\delta)}{N}\right)$ points of $\{Y^*_1, ..., Y^*_N\}$.

Then, the univariate trimmed estimator $\hat\mu_R (\{X_1, Y_2, ..., X_{N}, Y_N\}, \delta)$ can be given by
\begin{equation}
    \hat\mu_R = \frac{1}{\sum_{i}^N \mathbbm{1}\left\{X_i \in \mathcal Z\right\}}\sum_{i=1}^N X_i\mathbbm{1}\left\{X_i \in \mathcal Z\right\}.
\end{equation}
\end{definition}
The fundamental idea behind the trimmed mean is to use half of the input points to identify quartiles, and then use the remaining half to estimate the ``trimmed'' mean around those quartiles. A natural question arises regarding the optimality of this estimator and its rate of concentration around the true mean of $P$. Recent work from~\citep{prasad2019unified} provides such a result.
\begin{theorem}[Confidence bound for the Trimmed Mean~\citep{prasad2019unified}]
Let $\delta \in (0, 1)$. Then for any distribution $P$ with mean $\mu$ and finite variance $\sigma^2$, we have, with probability at least $1-\delta$,
\begin{equation*}
    \left|\hat\mu_R - \mu \right| \leq \sigma\sqrt{\epsilon} + \sqrt{\frac{\sigma\ln(1/\delta)}{N}}.
\end{equation*}
\label{thm:robust_confidence_bound}
\end{theorem}
Using these results, we can design an algorithm for the byzantine cooperative bandit problem, as described next.
\subsection{Byzantine-Proof Multi-Agent UCB}
The byzantine-proof algorithm works for an agent $m$ by making a conservative assumption that \text{all} agents (including itself) are byzantine, and constructing a robust mean estimator under this assumption. It then uses this estimator and an associated upper confidence bound (UCB) to choose an arm, similar to the single-agent UCB algorithm. The assumption of all agents being byzantine primarily is made to aid in the analysis of the algorithm: from the perspective of any agent, if all $M$ agents are byzantine, then it can expect to obtain approximately $O(TM\epsilon)$ corrupted messages. If however, only a fraction $f < 1$ of agents are byzantine, then it can expect to obtain approximately $O(TMf\epsilon)$ corrupted messages, which would imply that (in expectation), all $M$ agents are byzantine with $\epsilon' = f\epsilon$. This information can be incorporated at runtime as well, and hence we proceed with the conservative assumption. Moreover, if we overestimate $\epsilon$ during initalization, the performance of the algorithm will remain unchanged, since the robust mean estimator for any value $\epsilon$ is also applicable to any other mixture proportion $\epsilon' < \epsilon$ without any degradation in performance.

In this setting, for the our description of the algorithm, we present both the algorithm and the message-passing protocol in the same setup, unlike the previous case since there is no explicit distinction to be made between the decision-making algorithm and the message-passing algorithm, as they both operate together. The complete algorithm for all agents is described in Algorithm~\ref{alg:bp_multi_ucb}.
\begin{algorithm}[t]
\caption{\textsc{Byzantine-Proof Multi-Agent UCB}}
\label{alg:bp_multi_ucb}
\begin{algorithmic}[1] 
\STATE \textbf{Input}: Agent $m$, arms $k \in [K]$, mean estimator $\hat\mu_R(n, \delta)$
\STATE Set $S^m_k = \phi\ \forall k \in [K]$, $Q_m(t) = \phi$.
\FOR{For $t \in [T]$}
\IF{$t \leq K$}
\STATE $A_{m, t} = t$.
\ELSE
\FOR{Arm $k \in [K]$}
\STATE $\hat\mu^{(m)}_k = \hat\mu_R(S^m_k, 1/t^2)$.
\STATE $\text{UCB}_k^{(m)}(t) = \sigma\sqrt{\epsilon} + \sqrt{\frac{\sigma\ln(1/\delta)}{|S^m_k(t)|}}$.
\ENDFOR
\STATE $A_{m, t} = \arg\max_{k \in [K]} \left\{\hat\mu_k ^{(m)}(t) + \text{UCB}_k^{(m)}(t)\right\}$.
\ENDIF
\STATE $X_{m, t} = \textsc{Pull}(A_{m, t})$.
\STATE $S_{A_{m,t}}^m = S_{A_{m,t}}^m \cup \{X_{m, t}\}$
\STATE $Q_m(t)$ = \textsc{PruneDeadMessages}($Q_m(t)$).
\STATE $Q_m(t) = Q_m(t) \cup \{\left\langle m, t, \gamma, A_{m, t}, X_{m, t}\right\rangle\}$.
\STATE Set $l =l -1 \ \forall \langle m',t', a', x'\rangle$ in $Q_m(t)$.
\FOR{Each neighbor $m'$ in $\mathcal N_1(m)$}
\STATE \textsc{SendMessages}$(m, m', Q_m(t))$.
\ENDFOR
\STATE $Q_m(t+1) = \phi$.
\FOR{Each neighbor $m'$ in $\mathcal N_1(m)$}
\STATE $Q' = $\textsc{ReceiveMessages}$(m',m)$
\STATE $Q_m(t+1)$ = $Q_m(t+1) \cup Q'$.
\ENDFOR
\FOR{$\langle m',t',a',x'\rangle \ \in Q_m(t+1)$}
\STATE $S_{a'}^m = S_{a'}^m \cup \{x'\}$.
\ENDFOR
\ENDFOR
\end{algorithmic}
\end{algorithm}

\textbf{Algorithm Description}. We will consider an agent $m \in [M]$. The agent maintains a set of rewards $S_k^m$ for each arm $k$, which it updates with the reward samples it obtains from messages from other agents, and its own rewards. At every trial, for each arm, it uses the robust mean estimator $\hat\mu_R$ to compute the trimmed mean of all the reward samples present in $S_k^m$, and uses the number of samples to estimate an upper confidence bound as well. It then chooses the arm with the largest sum of the (robust) upper confidence bound. After choosing an arm, it updates its set $S_k^m$ and sends messages to all its neighbors, similar to the previous algorithm. 
\begin{theorem}
If all agents $m \in [M]$ run algorithm~\ref{alg:bp_multi_ucb} with mean estimator $\hat\mu_R$, then for $\epsilon < \Delta_{\min}/2\sigma$, the group regret after $T$ iterations obeys the following upper bound.
\begin{multline*}
    R_\mathcal G(T) \leq \bar\chi\left(\mathcal G_\gamma\right)\left(\sum_{k : \Delta_k > 0} \frac{4\sigma^2}{(\Delta_k - 2\sigma\sqrt{\epsilon})^2}\right)\ln T +\\ \left(3M + \gamma\chi\left({\mathcal G_\gamma}\right)\left(M-1\right)\right)\left(\sum_{k : \Delta_k > 0}\Delta_k\right).
\end{multline*}
Here, $\bar\chi(\cdot)$ refers to the minimum clique number.
\end{theorem}
\begin{proof}(Sketch). We proceed in a manner similar to the previous regret bound, by decomposing the group regret into a sum of clique-wise group regrets. We then construct similar UCB-like ``exploration'' vs ``exploitation'' events, and bound the probability of pulling an incorrect arm under these events using the confidence bounds of the robust mean estimator coupled with the worst-case delay within the network. The full proof is in the Appendix.
\end{proof}
\begin{remark}[Deviation from Optimal Rates]
When comparing the group regret bound to the optimal rate achievable in the single-agent case, there is an additive constant that arises from the delay in the network. Identical to the differentially private algorithm, this additive constant reduces to the constant corresponding to $MT$ individual trials of the UCB algorithm when information flows instantly throughout the network, i.e. $\mathcal G$ is connected. Additionally, we observe identical dependencies on other graph parameters in the leading term as well, except for the modified denominator $(\Delta_k - 2\sigma\sqrt\epsilon)^2$. This term arises from the inescapable bias that the mixture distribution $Q$ introduces into estimation, and~\citep{prasad2019unified} show that this bias is optimal and unimprovable in general. This can intuitively be explained by the fact that after a certain $\epsilon$, there will be so many corrupted samples that it would be impossible to distinguish the best arm from the next best arm, regardless of the number of samples (since the observed, noisy distributions for each arms will become indistinguishable). We must also note that the dependence on $\gamma$ on the asymptotic optimality remains identical to the previous algorithm, and when all agents are truthful ($\epsilon = 0$) and $\gamma = \textsf{diameter}(\mathcal G)$, this algorithm obtains optimal multi-agent performance, matching the lower bound.
\end{remark}
\begin{remark}[Simultaneous Privacy and Robustness]
A natural question to ask is whether both privacy with respect to other agents' rewards and robustness with respect to byzantine agents can be simultaneously achieved. This is a non-trivial question, since the fundamental approaches to both problems are conflicting. First, we must notice that to achieve differential privacy without damaging regret, one must convey summary statistics (such as the mean or sum of rewards) directly, with appropriately added noise. However, for robust estimation in byzantine-tolerance, communicating individual rewards is essential, to allow for an appropriate rate of concentration of the robust mean. These two approaches seem contradictory at a first inquiry, and hence we leave this line of research for future work.
\label{remark:doing_both}
\end{remark}
In the next section, we decribe the experimental comparisons of our proposed algorithms with existing benchmark algorithms.

\begin{figure*}[t]
  \includegraphics[width=\linewidth]{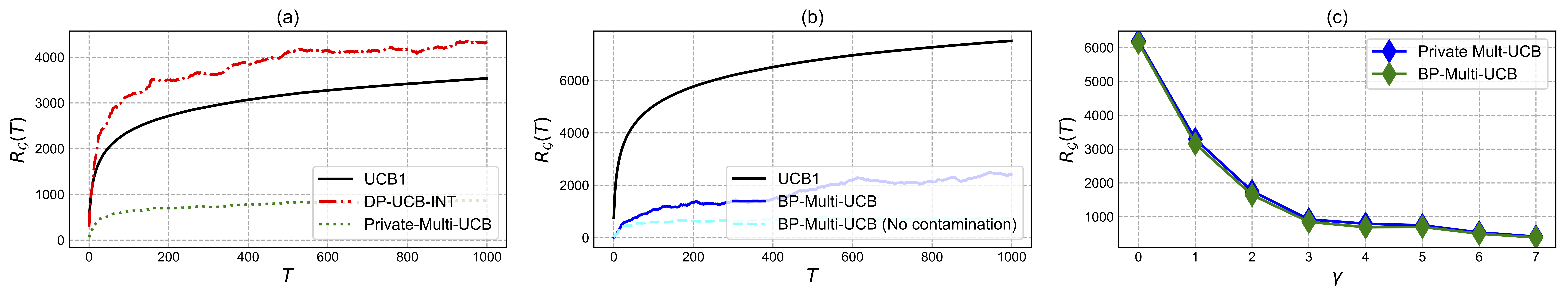}
  \caption{Experimental comparisons on random graphs. Each figure is constructed by averaging over 100 trials.}
  \label{fig:main}
\end{figure*}
\section{Experiments}
For our experimental setup, we consider rewards drawn from randomly initialized Bernoulli distributions, with $K=5$ as our default operating environment. We initialize the connectivity graph $\mathcal G$ as a sample from a random Erdos-Renyi~\citep{bollobas2003mathematical} family with edge probability $p=0.1$, and set the communication parameter $\gamma = \textsf{diameter}(\mathcal G)/2$. We describe the experiments and associated benchmark algorithms individually.

\textbf{Testing Private Multiagent UCB}. In this setting, we compare the group regret $R_\mathcal G$ of $M=200$ agents over $100$ randomly initialized trials of the above setup. The benchmark algorithms we compare with are (a) single-agent UCB1~\citep{auer2002finite} running individually on each agent, and (b) DP-UCB-INT of~\citep{tossou2016algorithms} running individually on each agent, under the exact setup for $\epsilon$ and $\delta$ as described in~\citep{tossou2016algorithms}. We choose this benchmark as it is the state-of-the-art in the single-agent stochastic bandit case. The results of this experiment are plotted in Figure~\ref{fig:main}(a). We observe that the performance of our algorithm is significantly better than both private benchmarks, however, it incurs higher regret (as expected) than the non-private version.

\textbf{Testing Byzantine-Proof Multiagent UCB}. To test the performance of our byzantine-proof algorithm, we use the same setup as the previous case, with $M=200$ agents, repeated over $100$ randomly initialized trials, except we set the reward probability of each arm to be constrained within $[0.3, 0.7]$, and set the contamination probability as $\epsilon = 10^{-3}$. The benchmark algorithms in this case are (a) the single-agent UCB1~\citep{auer2002finite} running individually on each agent, and (b) the byzantine-proof multiagent UCB with $\epsilon=0$ (since no other work explicitly studies the Huber $\epsilon$-contamination model as used in our setting). The results of this experiment are summarized in Figure~\ref{fig:main}(b). Similar to the previous case, we observe a large improvement in the group regret compared to the single-agent UCB1, and with no contamination, the performance is better.

\textbf{Testing the effect of $\gamma$}. To understand the effect of $\gamma$ on the obtained regret, we repeated the same experiment ($M=200$, 100 trials of randomly generated Erdos-Renyi graphs with $p=0.1$) with our two algorithms and compared their obtained group regret at $T=1000$ trials. We observe a sharp decline as $\gamma$ increases from $1$ to $\textsf{diameter}(\mathcal G)$, and matches the optimal group regret at $\gamma = \textsf{diameter}(\mathcal G)$, as hypothesized by our regret bound. The results of this experiment are summarized in Figure~\ref{fig:main}(c).
\section{Related Work}
Our work relates to a vast body of work across several fields in the sequential decision-making literature. 

\textbf{Multi-Agent Bandit Algorithms}. Multi-agent bandit problems have been studied widely in the context of distributed control~\citep{liu2010decentralized, liu2010distributed2, liu2010distributed, bistritz2018distributed}. In this line of work, the problem is competitive rather than cooperative, where a group of agents compete for a set of arms, and collisions occur when two or more agents select the same arms. Contextual bandit algorithms have also been explored the context of peer-to-peer networks and collaborative filtering over multiple agent recommendations~\citep{korda2016distributed, gentile2014online, gentile2017context, cesa2013gang}, where ``agents'' are clustered or averaged by their similarity in contexts. While some of these problem domains assume an underlying network structure between agents, none assume ``delay'' in the observed feedback between agents, and communication is instantaneous with the knowledge of the graph, making these algorithms effectively single-agent. For the cooperative bandit problem, however, general settings have been investigated in both the stochastic case using a \textit{consensus} protocol~\citep{landgren2016distributed, landgren2016distributed2, landgren2018social} and in the non-stochastic case using a message-passing protocol similar to ours~\citep{cesa2019delay, bar2019individual}. To the best of our knowledge, our work is the first to consider the cooperative bandit problem with differential privacy or byzantine agents.

\textbf{Private Sequential Decision-Making}. Our work uses the immensely influential \textit{differential privacy} framework of Dwork~\citep{dwork2011differential}, that has been the foundation of almost all work in private sequential decision-making. For the bandit setting, there has been significant research interest in the~\textit{single-agent} case, including work on private stochastic bandit algorithms~\citep{tossou2016algorithms, mishra2015nearly, thakurta2013nearly}, linear bandit algorithms~\citep{shariff2018differentially} and in the nonstochastic case~\citep{tossou2017achieving}. In the competitive multi-agent stochastic bandit case, Tossou and Dimitrakakis~\citep{tossou2015differentially} provide a UCB-based algorithm based on Time-Division Fair Sharing (TDFS). 

\textbf{Robust Sequential Decision-Making}. Robust estimation has a rich history in the bandit literature. Robustness to heavy-tailed reward distributions has been extensively explored in the stochastic multi-armed setting, from the initial work of Bubeck~\textit{et al.}\citep{bubeck2013bandits} that proposed a UCB algorithm based on robust mean estimators, to the subsequent work of~\citep{yupure, vakili2013deterministic, dubey2019thompson} on both Bayesian and frequentist algorithms for the same. Contamination in bandits has also been explored in the stochastic \textit{single-agent} case, such as the work of~\citep{altschuler2019best} in best-armed identification under contamination and algorithms that are jointly optimal for the stochastic and nonstochastic case~\citep{seldin2014one}, that uses a modification of the popular EXP3 algorithm. Our work builds on the work in robust statistics of Huber~\citep{huber2011robust} and Lugosi and Mendelson~\citep{lugosi2019robust}.

To the best of our knowledge, our work is the first to consider privacy and contamination in the context of the multi-agent cooperative bandit problem, with the additional benefit of being completely decentralized.
\section{Conclusion}
In this paper, we discussed the cooperative multi-armed stochastic bandit problem under two important practical settings -- private communication and the existence of byzantine agents that follow an $\epsilon$-contamination model. We provided two algorithms that are completely decentralized, and provide optimal group regret guarantees when run with certain parameter settings. Our work is the first to investigate real-world scenarios of cooperative decision-making, however, it does leave many open questions for future work.

First, we realise that achieving both differential privacy and byzantine-tolerance simultaneously is non-trivial, even in the stochastic case: differential privacy requires the computation of summary statistics (such as the sum or average of rewards) in order to provide feasible regret, however, for byzantine-tolerance we require the use of robust estimators, that explicitly require individual reward samples to compute, making their combination a difficult problem, which can be investigated in future work. Next, we analysed a very specific setting for byzantine agents, which can be generalized to adversarial corruptions, as done in the single-agent case. Finally, extensions of the stochastic multi-armed setting to contextual settings is also an important direction of research that our work opens up in multi-agent sequential decision-making.


\bibliographystyle{ACM-Reference-Format}  
\bibliography{sample-bibliography}  
\newpage
\onecolumn
\section*{Appendix}
\begin{theorem}
If all agents $m \in [M]$ each follow the \textsc{Private Multi-Agent UCB} algorithm with the messaging protocol described in Algorithm~\ref{alg:privacy_protocol}, then the group regret incurred after $T$ trials obeys the following upper bound.
\begin{align*}
    R_{\mathcal G}(T) \leq \sum_{k : \Delta_k > 0} \chi(\mathcal G_\gamma)\left(\frac{8\sigma^2\ln T}{\Delta_k}\right) + \left(\sum_{k : \Delta_k > 0} \Delta_k\right)\left(\chi(\mathcal G_\gamma)\left(M\gamma + 2\right) + M\left(\frac{1}{\epsilon} + \zeta(1.5)\right)\right).
\end{align*}
Here, $\chi$ is the clique number, and $\zeta$ is the Riemann zeta function.
\end{theorem}
\begin{proof}
Let a clique covering of $\mathcal G_\gamma$ be given by $\bf{C}_\gamma$. We begin by decomposing the group regret.
\begin{align}
    R_{\mathcal G}(T) &= \sum_{m=1}^M R_m(T) \\
    &\leq \sum_{\mathcal C \in \bf{C}_\gamma} \sum_{m \in \mathcal C} \sum_{k=1}^K \Delta_{k}\mathbb E[n_m^k(T)] \\
    &= \sum_{\mathcal C \in \bf{C}_\gamma}\sum_{k=1}^K \Delta_k \left(\sum_{m\in\mathcal C}\sum_{t=1}^T \mathbb P\left(A_{m, t} = k\right)\right) \\ \intertext{Consider the cumulative regret $R_{\mathcal C}(T)$ within the clique $\mathcal C$. For some time $T^k_{\mathcal C}$, assume that each agent has pulled arm $k$ for $\eta_m^k$ trials, where $\eta_{\mathcal C}^k = \sum_{m \in \mathcal C}\eta_m^k$. Then,}
    R_{\mathcal C}(T)&\leq \sum_{k=1}^K \Delta_k \left(\eta_{\mathcal C}^k + \sum_{m \in \mathcal C} \sum_{t=T^k_\mathcal C}^T \mathbb P\left(A_{m, t} = k, N_k^\mathcal C(t) \geq \eta^k_\mathcal C \right)\right). \label{eqn:regret_bare}
\end{align}
Here $N_k^{\mathcal C}(t)$ denotes the number of times arm $k$ has been pulled by any agent in $\mathcal C$. We now examine the probability of agent $m \in \mathcal C$ pulling arm $k$. First note that the empirical mean of any arm can be given by the latest messages accumulated by agent $m$ until that time. This can be given by the following, for any arm $k \in [K]$.
\begin{align}
    \hat\mu_k^m(t-1) &= \sum_{m' \in \mathcal N(m)} \left(\frac{\sum_{u=1}^{n_{m'}^k(t-d(m, m'))} X^k_{m', u}}{n_{m'}^k(t-d(m, m'))} + Y^k_{m'} \right) \\ \intertext{Here, $Y^k_{m'} \sim \mathcal L\left( n_{m'}^k(t-d(m, m'))^{v/2-1}\right)$. For convenience, let's denote the noise-free mean $\hat\mu_k^m(t-1) - \sum_{m' \in \mathcal N(m)} Y^k_{m'}$ as $Z_k^m(t-1)$, and $N_k^m(t) = n_m^k(t) + \sum_{m' \in \mathcal N(m)} n_{m'}^k(t-d(m, m'))$. Note that an arm is pulled when one of three events occurs:}
    \text{Event (A): } &  Z_k^m(t-1) \leq \mu_* - \sigma\sqrt{\frac{2\ln t}{N_*^m(t)}} -  \sum_{m' \in \mathcal N(m)} Y^*_{m'} \\
    \text{Event (B): } &  Z_k^m(t-1) \geq \mu_k + \sigma\sqrt{\frac{2\ln t}{N_k^m(t)}} + \sum_{m' \in \mathcal N(m)} Y^k_{m'} \\
    \text{Event (C): } &  \mu_* \leq \mu_k + 2\sigma\sqrt{\frac{2\ln t}{N_k^m(t)}}
\end{align}
We will first analyse events $(A)$ and $(B)$. We know from Dwork and Roth~\citep{dwork2014algorithmic} that for any $N$ random variables $Y_i \sim \mathcal L(b_i)$,
\begin{equation}
    \text{Pr}\left(\left|\sum_i Y_i\right| \geq \ln\left(\frac{1}{\omega}\right)\sqrt{8\sum_{i}b_i^2}\right) \leq \omega
\end{equation}
For some $\omega \in (0, 1)$, let $h_{k, m}(t) = \ln\left(\frac{1}{\omega}\right)\sqrt{8\sum_{m' } (n_{m'}^k(t-d(m, m'))^{v-2}}$. Let us examine the probability of event (A) occuring.
\begin{align}
    \text{Pr}(B) &= \text{Pr}\left(Z_k^m(t-1) \geq \mu_k + \sigma\sqrt{\frac{2\ln t}{N_k^m(t)}} + \sum_{m' \in \mathcal N(m)} Y^k_{m'}\right)\\
    &= \text{Pr}(\hat\mu_k^m(t-1) - Z_k^m(t-1) \geq h_{k, m}(t) \And Z_k^m(t-1) \geq \mu_k + \sigma\sqrt{\frac{2\ln t}{N_k^m(t)}} - h_{k, m}(t)) \\
    &\leq \text{Pr}\left(\hat\mu_k^m(t-1) - Z_k^m(t-1) \geq h_{k, m}(t)\right) +  \text{Pr}\left( Z_k^m(t-1) \geq \mu_k + \sigma\sqrt{\frac{2\ln t}{N_k^m(t)}} - h_{k, m}(t)\right) \\
    &\leq \omega + \text{Pr}\left( Z_k^m(t-1) \geq \mu_k + \sigma\sqrt{\frac{2\ln t}{N_k^m(t)}} - h_{k, m}(t)\right) \\
    &\leq \omega + \exp\left(-2N_k^m(t)\left(\sigma\sqrt{\frac{2\ln t}{N_k^m(t)}} - h_{k, m}(t)\right)^2\right) \label{eqn:step1} \\
    &= t^{-3.5} + \exp\left(-2N_k^m(t)\left(\sigma\sqrt{\frac{2\ln t}{N_k^m(t)}} - 3.5\ln t\sqrt{8\sum_{m' } (n_{m'}^k(t-d(m, m'))^{v-2}}\right)^2\right) \label{eqn:step2}\\
    &\leq t^{-3.5} + \exp\left(-2N_k^m(t)\left(\sigma\sqrt{\frac{2\ln t}{N_k^m(t)}} - 3.5\ln t\sqrt{8M}\left(\epsilon^{-1} - \gamma\right)^{v/2-1}\right)^2\right) \label{eqn:step3}\\
     &\leq 2t^{-3.5} \label{eqn:step4}
\end{align}
Here, we use Hoeffding's Inequality in Equation~(\ref{eqn:step1}), and the fact that $v \in (1, 1.5)$ in Equation~(\ref{eqn:step3}) and that $n_m^k(t) \geq \epsilon^{-1}$, and choose $\sum_{m'} n^k_{m'}(t-d(m, m'))$ such that $$\exp\left(-2N_k^m(t)\left(\sigma\sqrt{\frac{2\ln t}{N_k^m(t)}} - 3.5\ln t\sqrt{8M}\left(\epsilon^{-1} - \gamma\right)^{v/2-1}\right)^2\right) \leq t^{-3.5}.$$ 
We see that as long as $N_k^{\mathcal C}(t) \geq \frac{(\epsilon^{-1} -\gamma)^{v/2 -1}(2\sqrt{2}\sigma - \sqrt{3.5})}{\sqrt{24M\ln t}} + \frac{|\mathcal C|}{\epsilon}$, this is true, following the fact that $\gamma < 1/\epsilon$ and the argument presented in~\citep{tossou2015differentially} for the single agent case. We can repeat the same process for the event (A). Finally, let us analyse event (C). For (C) to be true, we must have the following to be true.
\begin{align}
    N_k^m(t) < \frac{8\sigma^2\ln t}{\Delta_k^2}.
\end{align}
Since $N_k^m(t) \geq N_k^{\mathcal C} - M\gamma$, we see that this event does not happen for any agent $m \in \mathcal C$ if we set $$\eta_k^{\mathcal C} = \ceil*{\max\left\{\frac{8\sigma^2\ln T}{\Delta_k^2} + M\gamma, \frac{(\epsilon^{-1} -\gamma)^{v/2 -1}(2\sqrt{2}\sigma - \sqrt{3.5})}{\sqrt{24M\ln t}} + \frac{|\mathcal C|}{\epsilon}\right\}}.$$
Next, we should notice that the second term decreases as $t$ increases. We therefore, can decompose the regret of the entire clique following the argument in~\citep{tossou2015differentially} identically.
\begin{align}
    R_{\mathcal C}(T)&\leq \sum_{k=1}^K \Delta_k \left(\eta_{\mathcal C}^k + \frac{|\mathcal C|}{\epsilon} + \sum_{m \in \mathcal C} \sum_{t=T^k_\mathcal C}^T \mathbb P\left(A_{m, t} = k, N_k^\mathcal C(t) \geq \eta^k_\mathcal C \right)\right) \\
    &\leq \sum_{k : \Delta_k > 0} \Delta_k \left(\frac{8\sigma^2\ln T}{\Delta_k^2} + M\gamma + 2 + \frac{|\mathcal C|}{\epsilon} + \sum_{m \in \mathcal C} \sum_{t=1}^T 4t^{-1.5}\right) \\
    &\leq \sum_{k : \Delta_k > 0} \Delta_k \left(\frac{8\sigma^2\ln T}{\Delta_k^2} + M\gamma + 2 + \frac{|\mathcal C|}{\epsilon} + \sum_{m \in \mathcal C} \sum_{t=1}^T 4t^{-1.5}\right) \\
    &\leq \sum_{k : \Delta_k > 0} \Delta_k \left(\frac{8\sigma^2\ln T}{\Delta_k^2} + M\gamma + 2 + |\mathcal C|\left(\frac{1}{\epsilon} + \zeta(1.5)\right)\right) \\ \intertext{Summing over all cliques $\mathcal C \in \bf{C}_\gamma$, we get}
    R_{\mathcal G}(T) &\leq \sum_{C \in \mathcal C_\gamma}\sum_{k : \Delta_k > 0} \Delta_k \left(\frac{8\sigma^2\ln T}{\Delta_k^2} + M\gamma + 2 + |\mathcal C|\left(\frac{1}{\epsilon} + \zeta(1.5)\right)\right).
\end{align}
Choosing $\bf{C}_\gamma$ to be the minimal clique partition of $\mathcal G_\gamma$,we obtain the final form of the bound.
\end{proof}
\begin{theorem}
If all agents $m \in [M]$ run algorithm~\ref{alg:bp_multi_ucb} with mean estimator $\hat\mu_R$, then for $\epsilon < \Delta_{\min}/2\sigma$, the group regret after $T$ iterations obeys the following upper bound.
\begin{align*}
    R_\mathcal G(T) \leq \chi\left(\mathcal G_\gamma\right)\left(\sum_{k : \Delta_k > 0} \frac{4\sigma^2}{(\Delta_k - 2\sigma\sqrt{\epsilon})^2}\right)\ln T +\left(3M + \gamma\chi\left({\mathcal G_\gamma}\right)\left(M-1\right)\right)\left(\sum_{k : \Delta_k > 0}\Delta_k\right).
\end{align*}
Here, $\chi(\cdot)$ refers to the minimum clique number.
\end{theorem}
\begin{proof}
Let a clique covering of $\mathcal G_\gamma$ be given by $\bf{C}_\gamma$. We begin by decomposing the group regret.
\begin{align}
    R_{\mathcal G}(T) &= \sum_{m=1}^M R_m(T) \\
    &\leq \sum_{\mathcal C \in \bf{C}_\gamma} \sum_{m \in \mathcal C} \sum_{k=1}^K \Delta_{k}\mathbb E[n_m^k(T)] \\
    &= \sum_{\mathcal C \in \bf{C}_\gamma}\sum_{k=1}^K \Delta_k \left(\sum_{m\in\mathcal C}\sum_{t=1}^T \mathbb P\left(A_{m, t} = k\right)\right) \\ \intertext{Consider the group regret $R_{\mathcal C}(T)$ within the clique $\mathcal C$. For some time $T^k_{\mathcal C}$, assume that each agent has pulled arm $k$ for $\eta_m^k$ trials, where $\eta_{\mathcal C}^k = \sum_{m \in \mathcal C}\eta_m^k$. Then,}
    R_{\mathcal C}(T)&\leq \sum_{k=1}^K \Delta_k \left(\eta_{\mathcal C}^k + \sum_{m \in \mathcal C} \sum_{t=T^k_\mathcal C}^T \mathbb P\left(A_{m, t} = k, N_k^\mathcal C(t) \geq \eta^k_\mathcal C \right)\right). \label{eqn:regret_bare2}
\end{align}
Here $N_k^{\mathcal C}(t)$ denotes the number of times arm $k$ has been pulled by any agent in $\mathcal C$. We now examine the probability of agent $m \in \mathcal C$ pulling arm $k$. Note that an arm is pulled when one of three events occurs:
\begin{align}
    \text{Event (A): } &  \hat\mu_*^m(t-1) \leq \mu_* -  \sigma\sqrt{\epsilon} - \sqrt{\frac{\sigma\ln(1/\delta)}{|S^m_*(t)|}} \\
    \text{Event (B): } &  \hat\mu_k^m(t-1) \geq \mu_k + \sigma\sqrt{\epsilon} + \sqrt{\frac{\sigma\ln(1/\delta)}{|S^m_k(t)|}} \\
    \text{Event (C): } &  \mu^* \leq \mu^k + 2\sigma\sqrt{\epsilon} + 2\sqrt{\frac{\sigma\ln(1/\delta)}{|S^m_k(t)|}} 
\end{align}
Now, let us examine the occurence of event $(C)$:
\begin{align}
  \Delta_k -2\sigma\sqrt{\epsilon} \leq 2\sigma\sqrt{\epsilon} - \sqrt{\frac{\sigma\ln(1/\delta)}{|S^m_*(t)|}} \\
  \implies |S^m_k(t)| \leq \frac{4\sigma^2\ln t}{(\Delta_k - 2\sigma\sqrt\epsilon)^2}
\end{align}
Let us now examine a lower bound for $N_m^k(t) = |S_k^m(t)|$. Let $P_m^k(t)$ denote the set of reward samples obtained by agent $m$ for its own pulls of arm $k$ until time $t$. We know, then that $P_m^k(t) = P_m^k(t-1) \text{ if arm $k$ was pulled at time $t$, and }P_m^k(t) = P_m^k(t-1) \cup \{X_{m, t}\} \text{ otherwise.}$

Additionally, any message from an agent $m' \in \mathcal G$ takes $d(m, m') - 1$ iterations to reach agent $m$. Therefore:
\begin{align}
    S^m_k(t) = P_m^k(t) \cup \left\{ \bigcup_{m' \in \mathcal G \setminus \{m\}} P_m^k\left(t-d(m',m)+1\right)\right\}.
\end{align}
Note that $P_m^k(t)$ and $P_{m'}^{k'}(t')$ are disjoint for all $m \neq m', k, k', t, t'$. Let $n(S)$ denote the cardinality of $S$. Then,
\begin{align}
    n\left(S^m_k(t)\right) = n\left(P_m^k(t)\right) + \left\{ \sum_{m' \in \mathcal G \setminus \{m\}} n\left(P_m^k\left(t-d(m',m)+1\right)\right)\right\}.
\end{align}
Now, in the iterations $t-d(m, m')+1$ to $t$, agent $m'$ can pull arm $k$ at most $d(m, m') - 1$ times and at least 0 times. Therefore,
\begin{align}
    N_{\mathcal G_\gamma(m)}^k(t) \geq n\left(S^m_k(t)\right) &\geq \max\left\{0, N_{\mathcal G_\gamma(m)}^k(t) - \sum_{m' \in \mathcal G_\gamma(m) \setminus \{m\}}(d(m, m')-1)\right\}\\
    &\geq \max\left\{0, N_{\mathcal G_\gamma(m)}^k(t) + |\mathcal G_\gamma(m)|(1-\gamma)\right\}
\end{align}

Hence, $N_m^k(t) \geq N^{\mathcal C}_k(t) - (M - 1)(1-\gamma)$ for all $t$. Therefore, if we set $\eta^k_\mathcal C = \ceil*{\frac{4\sigma^2\ln T}{(\Delta_k - 2\sigma\sqrt\epsilon)^2} + (M - 1)(\gamma - 1)}$, we know that event $(C)$ will not occur. Additionally, using the union bound over $N^*_m(t)$ and $N^k_m(t)$, and Theorem~\ref{thm:robust_confidence_bound}, we have:
\begin{align}
    \mathbb P(\text{Event (A) or (B) occurs}) \leq 2\sum_{s=1}^t \frac{1}{s^4} \leq \frac{2}{t^3}.
\end{align}
Combining all probabilities, and inserting in Equation~(\ref{eqn:regret_bare}), we have,
\begin{align}
    R_{\mathcal C}(T)&\leq \sum_{k=1}^K \Delta_k \left(\eta_{\mathcal C}^k + \sum_{m \in \mathcal C} \sum_{t=T^k_\mathcal C}^T \mathbb P\left(A_{m, t} = k, N_k^\mathcal C(t) \geq \eta^k_\mathcal C \right)\right)\\
    &\leq \sum_{k=1}^K \Delta_k \left(\ceil*{\frac{4\sigma^2\ln T}{(\Delta_k - 2\sigma\sqrt\epsilon)^2} + (M - 1)(\gamma - 1)} + \sum_{m \in \mathcal C} \sum_{t=1}^T \frac{2}{t^3}\right) \\
    &\leq \sum_{k=1}^K \Delta_k \left(\ceil*{\frac{4\sigma^2\ln T}{(\Delta_k - 2\sigma\sqrt\epsilon)^2} + (M - 1)(\gamma - 1)} + 4|\mathcal C|\right) \\
    &\leq \sum_{k=1}^K \Delta_k \left(\frac{4\sigma^2\ln T}{(\Delta_k - 2\sigma\sqrt\epsilon)^2} + (M - 1)(\gamma - 1) + 1 + 4|\mathcal C|\right).
\end{align}
We can now substitute this result in the total regret.
\begin{align}
    R_\mathcal G(T) &\leq \sum_{\mathcal C \in \bf{C}_\gamma} R_\mathcal C(T) \\
    &\leq \sum_{\mathcal C \in \bf{C}_\gamma}\sum_{k=1}^K \Delta_k \left(\frac{4\sigma^2\ln T}{(\Delta_k - 2\sigma\sqrt\epsilon)^2} + (M - 1)(\gamma - 1) + 1 + 4|\mathcal C|\right).
\end{align}
Choosing $\bf{C}_\gamma$ as the minimum clique covering gets us the result.
\end{proof}
\end{document}